\documentclass{article}

\usepackage{mymacros}

\usepackage{PRIMEarxiv}

\usepackage{hyperref}       
\usepackage{url}            
\usepackage{amsfonts}       
\usepackage{graphicx}       
\usepackage{float}

\title{Sample size planning for conditional counterfactual mean estimation with a K-armed randomized experiment}

\author{
  Gabriel Ruiz\\
  Adobe Inc.\\
  San Jose, CA, U.S.A.\\
  \texttt{garuiz@adobe.com,17ruiz17@gmail.com} \\
}

\begin{document}

\maketitle

\begin{abstract}
We cover how to determine a sufficiently large sample size for a $K$-armed randomized experiment in order to estimate conditional counterfactual expectations in data-driven subgroups. The sub-groups can be output by any feature space partitioning algorithm, including as defined by binning users having similar predictive scores or as defined by a learned policy tree. After carefully specifying the inference target, a minimum confidence level, and a maximum margin of error, the key is to turn the original goal into a simultaneous inference problem where the recommended sample size to offset an increased possibility of estimation error is directly related to the number of inferences to be conducted. Given a fixed sample size budget, our result allows us to invert the question to one about the feasible number of treatment arms or partition complexity (e.g. number of decision tree leaves). Using policy trees to learn sub-groups, we evaluate our nominal guarantees on a large publicly-available randomized experiment test data set. 

\end{abstract}


\tableofcontents

\section{Introduction}

\begin{figure}[H]
    \centering
    \includegraphics[width=0.47\textwidth]{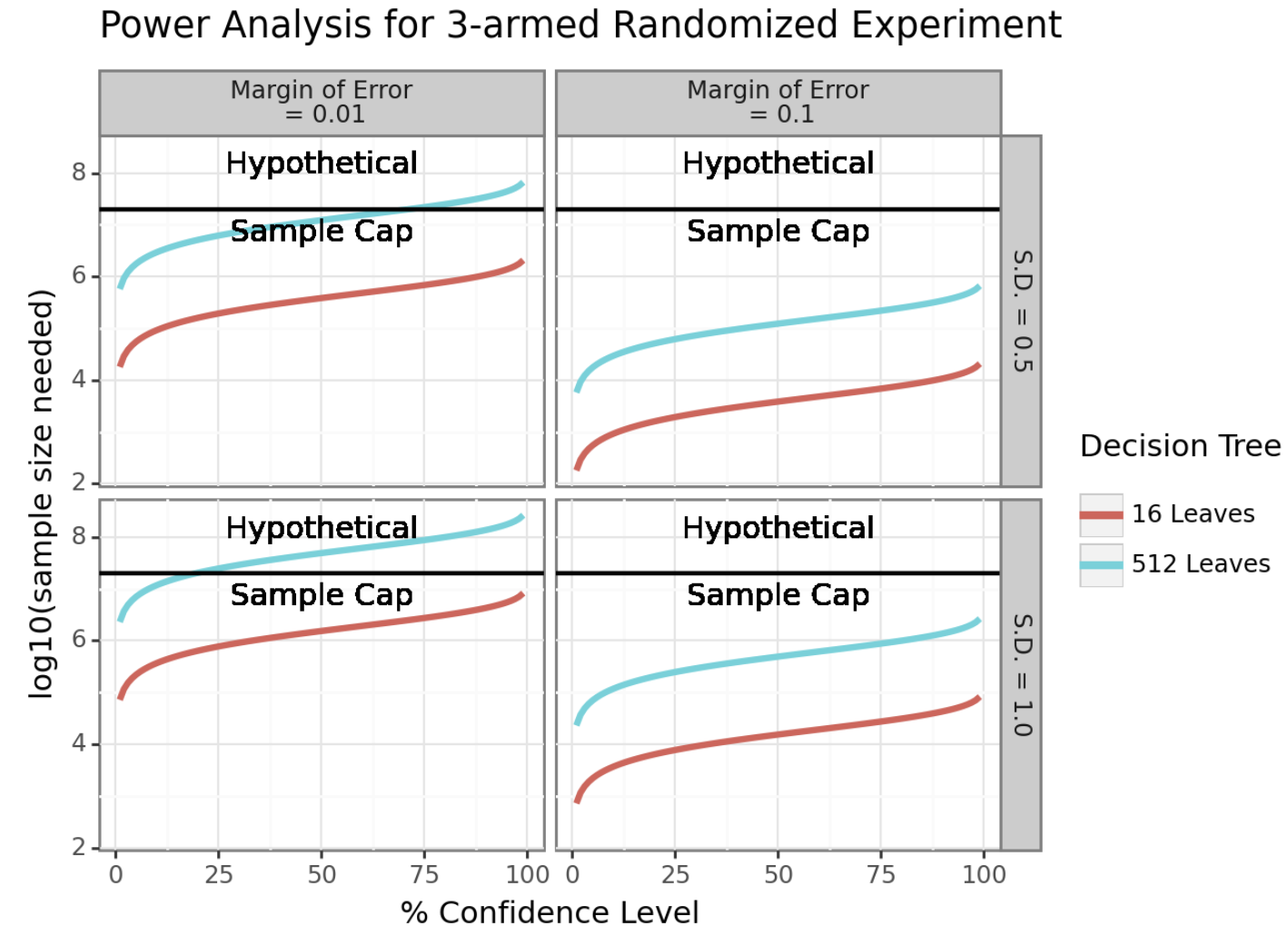}
    \caption{Example power analysis showing a sufficient sample size needed for jointly inferring the expected potential outcome in 3 treatment arms for a new randomly sampled individual. Here, individuals can fall into one of the sub-groups defined by a partition of the feature space given by an arbitrary decision tree. In addition to the number of sub-groups, the sample size is determined by a desired confidence level, margin of error, and unexplained outcome variation (or standard deviation). }
    \label{fig:sufficient}
\end{figure}

Randomized experiments (REs) are, in some sense, the gold standard for establishing causality. Although external validity is not guaranteed \citep{lei2023policy}, REs are a good starting point when they can be conducted. To this end, we consider the problem of sample size planning when we would like to contrast $K$ treatment groups. Noting the realistic possibility that the potential outcomes across possible treatments for one individual need not be similar to the potential outcomes of another individual, we consider sample size planning in the context of conditional counterfactual means, where the conditioning is with respect to baseline features. 

We note that reasoning about treatment effects at an individual-level is difficult, if not impossible \citep{cateVsITE2021}, even with an infinitely sized sample. For example, see \citet{ruizPadilla2022Pibt} for the limited statistical inference that is possible on the distribution of individual treatment effects. For a goal that is more immediately within reach, we turn our sample size planning problem into one that attempts to study counterfactuals based on their \textit{average} at the sub-group-level, such as defined by, but not limited to, a causal regression tree \citep{causalTree} or a policy tree \citep{policytree2021,poliTreeSofware}. We work under the flexibility that the sub-groups can be learned by \textit{any} machine learning model that partitions the space of baseline co-variates, or features, into disjoint subsets. We illustrate an example partition of the feature space using a policy tree \citep{policytree2021,poliTreeSofware} in Section \ref{sec:empiricalEval}.

In determining a sufficient sample size, we require a desired margin of error, frequentist confidence-level, model complexity--the number of subsets in the partition, and plausible bounds on outcome variation. In practice, this can be tuned to comply with realistic sample size constraints. Once these parameters are specified, the sufficient sample size needed requires that we enforce a minimum sample size constraint in each subset for each treatment group. To attain nominal guarantees with such a partition, we require honest prediction \citep{causalTree,grf2019} wherein outcomes for a holdout set in the training data are omitted from the partition splitting rules to be learned.

The key in our sample size planning is that in reasoning about the effect of $K$ treatments across such a partition, we are conducting a simultaneous inference of means. Given that mistakes are more likely with an increased number of simultaneous inferences, a sufficient sample size is needed to ensure a desired accuracy. We formalize the amount that is sufficient in Propositions \ref{prop:CLTBased}, \ref{prop:boundedOutcome}, and \ref{prop:boundedOutcomeBoundedVar}.

In the remainder of this section, we introduce our estimator and its super-population analogue. In section \ref{sec:methods}, we discuss the "high confidence" inference guarantee we are after, assumptions, and the main result that allows us to determine sample size as in the example shown in Figure \ref{fig:sufficient}. Given a fixed sample size budget, we discuss in \S~\ref{sec:fixedSampleSize} how our result allows us to invert the question to one about the feasible number of treatment arms or partition complexity (e.g. number of decision tree leaves). Next, Section \ref{sec:empiricalEval} empirically evaluates whether our main result holds using a semi-synthetic simulation from a large publicly available a/b test. Finally, Section \ref{sec:discussion} concludes with a discussion of limitations and possible extensions.

\subsection{Feature space partitioning model for conditional counterfactual means}

Suppose we have treatments $w=1,\dots,K$ we would like to contrast. We have data of baseline covariates, treatment exposure, and outcome tuples $\{(X_i,W_i,Y_i)\}_{i=1}^n$ with which we would like to approximate the non-trivial conditional counterfactual mean function,
\[
\mu^*(x,w)\ =\ \EC{Y_i(w)}{X_i=x}.
\]
Additionally, it can be of interest to study the pair-wise differences of these means, known as the conditional average treatment effect (CATE) function:
\[
\tau^*(x,w,w')\ =\ \EC{Y_i(w)-Y_i(w')}{X_i=x};\ w,w'\in\{1,\dots,K\}.
\]
Here, $\curl{Y_i(w),Y_i(w')}$ are known as potential outcomes \citep{rosenbaumRubinPropScore1983,imbens_rubin_2015}, and equivalently as counterfactuals \citep{pearl2000causality,hernan2023whatif}. These represent individual $i$'s hypothetical outcomes when we force, typically in a randomized experiment, the exposure variable $W_i$ to be $w$ or $w'$. 

Suppose we partition the feature space into $L\geq 1$  subsets
\[
\xcal_1,\dots,\xcal_L\ \subset\ \xcal.
\]

Denote:
\[
\ell:\ \xcal\to\{1,\dots,L\}
\]
as the mapping that assigns $\ell(x)=l$ if $x\in\xcal_l$. We will call $l=1,\dots,L$ a leaf, as it is helpful to think of the partition of $\xcal$ as arising from a decision tree \citep{Breiman2017-ay}. Denote the prediction indices in group $w=1,\dots,K$ in leaf $l=1,\dots,L$ as
\[
\ical(w,l)=\curl{i\in\scal:\ W_i=w,\ell(X_i)=l}\subset\{1,\dots,n\},
\]
and let its sample size be denoted by
\[
n_{wl}=\abs{\ical(w,l)}.
\]

Here, $\scal\subset\{1,\dots,n\}$ denotes the holdouts for honest prediction \citep{causalTree,grf2019}.

Making use of the partition of the feature space, suppose our approximation of $\mu^*(x,w)$ is given by
\[
\hat{\mu}(x,w)\mapsto\frac{1}{ n_{w\ell(x)} }\sum_{i\in\ical(w,\ell(x))}Y_i,
\]
the sample mean outcome in treatment group $w$ and leaf $\ell(x)$. Meanwhile, the approximation to $\tau^*(x,w,w')$, denoted as $\hat{\tau}(x,w,w')$, is given by:
\[
(x,w,w')\mapsto\hat{\mu}(x,w)-\hat{\mu}(x,w'),
\]
an estimator that is in strong analogy to, but which the partition mapping $\ell(x)$ need not be limited to, the causal decision trees of \citet{causalTree}.

\subsubsection{A manageable inference target in practice\label{sec:manageableInfTarget}}

Let $(\xnp,\wnp,\ynp)$ be a new independent draw such that \[\begin{aligned}
&\ynp\mid \wnp=w,\ell(\xnp)=\ell(x)\\
\iid&\ Y_i\mid W_i=w,\ell(X_i)=\ell(x);\ i=1,\dots,n.
\end{aligned}\]
While we mean to approximate the population-level quantity $\mu^*(x,w)$, the population-analogue of $\hat{\mu}(x,w)$ is actually
\[
\mu(x,w)= \EC{\ynp}{\wnp=w,\ell(\xnp)=\ell(x)}.
\]
This will be the target of our inference. 

Note that $\ical(w,l)$ is random because the tuples $(X_i,W_i)$ it is based on are random. Conditional on $\ical(w,\ell(x))$, we have that $\hat{\mu}(x,w)$ is unbiased for $\mu(x,w)$ by its construction:
\[\begin{aligned}
&\EC{\hat{\mu}(x,w)}{ \ical(w,\ell(x)) }\\
=\ &\frac{1}{n_{w\ell(x)} }\sum_{i\in\ical(x,w) }\EC{Y_i}{ W_i=w,\ell(X_i)=\ell(x),\ical(w,\ell(x)) }\\
=\ &\EC{\ynp}{\wnp=w,\ell(\xnp)=\ell(x)}.
\end{aligned}\]

The estimator $\hat{\mu}(x,w)$ is also unbiased for the conditional mean of an individual's counterfactual outcome
\[
\EC{\ynp(w)}{\ell(\xnp )=\ell(x)},
\]
provided the balancing score condition
\[
\{\ynp(w),\ynp(w')\}\indep \wnp|\ell(\xnp)=\ell(x)
\]
holds \citep{rosenbaumRubinPropScore1983,imbensReviewATE2004,imbens_rubin_2015, overlap2021,hernan2023whatif}. This condition holds in a randomized experiment setting when $X_i$ are pre-treatment features that are independent of treatment assignment. Importantly, this balancing condition with respect to the discrete partition mapping $\ell(x)$ need not hold outside of the experimental setting.

We similarly have that conditional on $\ical(w,\ell(x))$ and $\ical(w',\ell(x))$, $\hat{\tau}(x,w,w')$ is unbiased for \[
\tau(x,w,w')\ :=\ \mu(x,w)-\mu(x,w').
\]
The goal of our work will be to ensure the sets $\ical(w,l)$ across $(w,l)\in\curl{1,\dots,K}\times \curl{1,\dots,L}$ have sufficient indices to provide a desired high-confidence accuracy guarantee on $\mu(x,w)$. In addition to accurate inference on $\tau(x,w,w')$, we show such a guarantee also implies accurate inference on the best outcome across treatment arms, on average, in leaf $\ell(x)$: $\max_w\mu(x,w)$. We infer the latter quantity using its empirical analogue: $\max_{w}\hat{\mu}(x,w)$.


\section{Main result\label{sec:methods}}

We have the following key lemma before stating our main result. In it, denote
\[
\icaln = \curl{ \ical(w,l):\ (w,l)\in\curl{1,\dots,K}\times\curl{1,\dots,L} }
\]
as the partition of the honest set indices $\scal$ into the mutually disjoint subsets. 

\begin{lemma}[Implications of accurate conditional counterfactual mean estimation\label{lem:keyIneqs}]\ \\

Let $\alpha\in(0,1)$ and $\epsilon>0$. Consider a new random test point $X_{n+1}$.\\
If
\begin{equation}\label{eqn:accurateCondMeanRandomX}
\prc{\max_{ w }\abs{ \hat{\mu}( X_{n+1},w )-\mu(X_{n+1},w) } < \epsilon}{\icaln}\\
\geq\ 1-\alpha,
\end{equation}
then
\begin{equation}\label{eqn:bestMeanRandomX}\begin{aligned}
\ &\prc{\abs{ \max_{ w }\hat{\mu}( X_{n+1},w )-\max_{ w }\mu(X_{n+1},w) } < \epsilon}{\icaln}&\geq\ &1-\alpha
\end{aligned}\end{equation}
along with
\begin{equation}\label{eqn:guaranteeLeafMeanDiff}\begin{aligned}
&\prc{\max_{ w\neq w' }\abs{ \hat{\tau}( X_{n+1},w,w' )-\tau(X_{n+1},w,w') }<2\epsilon}{\icaln}&\geq\ &1-\alpha.
\end{aligned}\end{equation}

Moreover, if
\begin{equation}\label{eqn:accurateCondMeanSupX}
\prc{\sup_x\max_{ w }\abs{ \hat{\mu}( x,w )-\mu(x,w) } < \epsilon}{\icaln}\\
\geq\ 1-\alpha,
\end{equation}
then
\begin{equation}\label{eqn:bestMeanSupX}\begin{aligned}
\ &\prc{\sup_x\abs{ \max_{ w }\hat{\mu}( x,w )-\max_{ w }\mu(x,w) } < \epsilon}{\icaln}&\geq\ &1-\alpha
\end{aligned}\end{equation}
along with
\begin{equation}\label{eqn:guaranteeSupLeafMeanDiff}\begin{aligned}
&\prc{\sup_x\max_{ w\neq w' }\abs{ \hat{\tau}( x,w,w' )-\tau(x,w,w') }<2\epsilon}{\icaln}&\geq\ &1-\alpha.
\end{aligned}\end{equation}

\end{lemma}

Lemma \ref{lem:keyIneqs} is saying that an $(\epsilon,\alpha)$ estimation guarantee for each conditional counterfactual mean $\mu(w,\cdot)$, whether at a random point $\xnp$ or uniformly across $x$, implies an $(\epsilon,\alpha)$ estimation guarantee on the best (largest) conditional counterfactual treatment mean. Similarly, accurate estimation of the conditional counterfactual mean $\mu(w,\cdot)$ implies a $(2\epsilon,\alpha)$ guarantee for our estimation of the treatment effect function $\tau(\cdot,w,w')$. 

The proof of Lemma \ref{lem:keyIneqs} is contained in Appendix \ref{append:proofs}. The general idea in the proof is to use a law of total probability decomposition for the probabilities in \eqref{eqn:bestMeanRandomX} and \eqref{eqn:guaranteeLeafMeanDiff}
involving the maximal deviation
\[
\max_{ w }\abs{ \hat{\mu}( X_{n+1},w )-\mu(X_{n+1},w) }.
\]
We follow similar steps for the inequalities in \eqref{eqn:bestMeanSupX} and \eqref{eqn:guaranteeSupLeafMeanDiff}. 

Given Lemma \ref{lem:keyIneqs}, the goal is now to understand the sufficient sample size required per treatment group, per subset of the partition that grants us the premise in either Equation \eqref{eqn:accurateCondMeanRandomX} or \eqref{eqn:accurateCondMeanSupX}. Propositions \ref{prop:CLTBased}, \ref{prop:boundedOutcome}, and \ref{prop:boundedOutcomeBoundedVar} give us this understanding. 

\subsection{Some regularity conditions}
The required sample size depends on our assumption about the data generating mechanism. Some representative assumptions include the following.

\begin{assumption}[Bounded Outcome\label{assump:boundedOutcome}] For known constants $a<b$, suppose that $\pr{a\leq Y_i\leq b}=1$.
\end{assumption}

\begin{assumption}[Known bound on the conditional variance\label{assump:boundedVariance}] Let $\sigma^2>0$ be a known constant. Across $w=1,\dots,K$ and $l=1,\dots,L$, suppose that
\[
\sigma^2_{wl}:=\var{\ynp\mid\wnp=w,\ell(\xnp)=l}\leq\sigma^2.
\]

\end{assumption}

\begin{assumption}[Central Limit Theorem\label{assump:clt}]
Across $w=1,\dots,K$ and $x\in\xcal$, suppose that
\[
\hat{\mu}(x,w)\sim{\rm Normal}\parenth{\mu(x,w),\sigma^2_{w\ell(x)}/n_{wl}}.
\]

\end{assumption}

Section \ref{sec:boundCondVar} discusses strategies for obtaining a bound on the conditional variance. When bounds on the conditional variance are unknown, Remark \ref{rem:stdScale} below discusses how one can proceed with inference on a standardized scale.

\subsection{The main result\label{sec:mainResults}}

\begin{proposition}[Main Result: Central Limit Theorem\label{prop:CLTBased}]\ \\
Suppose Assumptions \ref{assump:boundedVariance} and \ref{assump:clt} hold. 

\begin{itemize}
\item Then for the guarantee in \eqref{eqn:accurateCondMeanRandomX} to hold, it is sufficient that
\begin{equation}\label{eqn:cltIneq}
\min_{w,l} n_{wl}\ \geq\ \parenth{\frac{\Phi^{-1}\parenth{(1-\alpha)^{ \frac{1}{K}}}\times \sigma  }{\epsilon}}^2
\end{equation}
with $\Phi^{-1}$ the inverse Cumulative Distrubiton Function (CDF) of the standard Normal distribution.
\item Then for the guarantee in \eqref{eqn:accurateCondMeanSupX} to hold, it is sufficient that 
\begin{equation}\label{eqn:cltIneqSup}
\min_{w,l} n_{wl}\ \geq\ \parenth{\frac{\Phi^{-1}\parenth{(1-\alpha)^{ \frac{1}{KL}}}\times \sigma  }{\epsilon}}^2.
\end{equation}

\end{itemize}

\end{proposition}

Using Proposition \ref{prop:CLTBased}, Figure \ref{fig:sufficient} plots the sufficient sample size as
\[
n\ =\ K\times L\times\frac{ \parenth{\frac{\Phi^{-1}\parenth{(1-\alpha)^{ \frac{1}{K}}}\times \sigma  }{\epsilon}}^2 }{\abs{\scal}/n}
\]
with the constraint that we use half the sample for honest prediction, i.e. $\abs{\scal}/n=0.5$. We include the example choices $\sigma\in\curl{0.5,1}$, $\epsilon\in\curl{0.01,0.1}$, $L\in\curl{16,512}$, and varying confidence levels $(1-\alpha)\times100\%$.

In the appendix, we include Propositions \ref{prop:boundedOutcome} and \ref{prop:boundedOutcomeBoundedVar}. The former works under Assumption \ref{prop:boundedOutcome}, while the latter works under Assumptions \ref{prop:boundedOutcome} and \ref{prop:boundedOutcomeBoundedVar}. These correspond to a more conservative alternative to Assumption \ref{assump:clt} \citep{Rosenblum2009-ue_concInequalities}. 

The proof of Propositions \ref{prop:CLTBased}, \ref{prop:boundedOutcome}, and \ref{prop:boundedOutcomeBoundedVar} is found in Appendix \ref{append:proofMainResults}. It is an application of Lemma \ref{lem:jointEstMeansClusters}, which estimates the means of $K$ groups based on independent data, across $L$ disjoint clusters. Conditional on the partition $\icaln$, the setting of Propositions \ref{prop:CLTBased}, \ref{prop:boundedOutcome}, and \ref{prop:boundedOutcomeBoundedVar} fits this mould. The exact sufficient sample size per tree per treatment group in each guarantee is given by the normality assumption, Hoeffding's inequality, and Bennett's inequality, respectively \citep{wainwright_2019,bennet1962,Rosenblum2009-ue_concInequalities}. 

Moreover, the statements in Propositions \ref{prop:CLTBased}, \ref{prop:boundedOutcome}, and \ref{prop:boundedOutcomeBoundedVar} with respect to a random new input $\xnp$ do not assume anything about the discrete distribution of the random variable $\ell(\xnp)$. This discrete distribution can be spiked, slabbed, skewed, or anything else\footnote{Inserting prior knowledge on, and constraining the posterior distribution of, this distribution could be an interesting future direction.}. The minimum sample size guarantees the nominal coverage, regardless. At the stage of sample size planning, this guard seems ideal as we may not have a good intuition for how a learned partition will distribute the Multinoulli label $\ell(\xnp)$.

\subsubsection{Applying Propositions \ref{prop:CLTBased}, \ref{prop:boundedOutcome}, and \ref{prop:boundedOutcomeBoundedVar} with a standardized scale}
In Section \ref{sec:boundCondVar}, we discuss how to obtain the bound on the conditional variance, $\sigma^2$ to satisfy Assumption \ref{assump:boundedVariance}. In case we do not have enough domain knowledge to obtain $\sigma^2$, we now discuss the consequence of applying Propositions \ref{prop:CLTBased}, \ref{prop:boundedOutcome}, and \ref{prop:boundedOutcomeBoundedVar} and Lemma \ref{lem:keyIneqs} with the specification $\sigma^2=1$. Remark \ref{rem:stdScale} explains that this specification turns our inference into one with a margin of error defined by a factor $\epsilon$ times an amount relating to the true conditional standard deviations.

Consider the outcome standardization
\[
Z_{iwl}\ :=\ \frac{Y_i}{\sigma_{wl}};\ i=1,\dots,n+1.
\]
With $Z_{iwl}$, we have the following standardized oracle estimator
\[
\hat{\mu}_{*}(x,w)\ =\ \frac{1}{n_{w \ell(x) }}\sum_{ i\in\ical(w,\ell(x)) }Z_{iw\ell(x)}\ =\ \frac{ \hat{\mu}(x,w) }{\sigma_{w \ell(x)} }
\]
for
\[
\mu_{*}(x,w)\ =\ \EC{ Z_{(n+1)wl} }{ \wnp=w,\ell(\xnp)=\ell(x) } \ =\ \frac{ \mu(x,w) }{\sigma_{w\ell(x)} }.
\]
Noting that $\varc{Z_{iwl} }{ W_i=w,\ell(X_i)=l,\icaln }= 1$, we have the following result with respect to our actual quantity of interest. 

\begin{remark}[Working on a standardized scale\label{rem:stdScale}]\ \\

To determine the sufficient sample size, we can apply Propositions \ref{prop:CLTBased}, \ref{prop:boundedOutcome}, and \ref{prop:boundedOutcomeBoundedVar} with respect to $\hat{\mu}_{*}(w,\cdot)$ and $\sigma\ \dot{=}\ 1$. The practical interpretation thanks to Lemma \ref{lem:keyIneqs} turns into:
\begin{enumerate}
\item $\abs{\hat{\mu}(\xnp,w)-\mu(\xnp,w)}<\epsilon\sigma_{w\ell(\xnp) }$ jointly for $w=1,\dots,K$ with probability $1-\alpha$;
\item $\abs{\max_w\hat{\mu}(\xnp,w)-\max_w\mu(\xnp,w)}<\epsilon \max_w\sigma_{w\ell(\xnp) }$ with probability $1-\alpha$;
\item $\abs{\hat{\tau}(\xnp,w,w')-\tau(\xnp,w,w')}<\epsilon\parenth{  \sigma_{w\ell(\xnp)}+\sigma_{w'\ell(\xnp) } }$ for each $w,w'$ with probability $1-\alpha$
\end{enumerate}
provided that the pertinent sample size constraint in Propositions \ref{prop:CLTBased}, \ref{prop:boundedOutcome}, and \ref{prop:boundedOutcomeBoundedVar} holds with $\sigma=1$. We have the analogous statements with respect to the guarantees uniformly across $x$ .

\end{remark}

Figure \ref{fig:sufficient} with $\sigma=1$ depicts the sufficient sample sizes for Remark \ref{rem:stdScale} with the choice $\epsilon\in\curl{0.01,0.1}$.

\subsection{Learning a feature space partition}

The partition
\[
\xcal_1,\dots,\xcal_L\ \subset\ \xcal.
\]
can be given by a decision tree \citep{Breiman2017-ay,grf2019,policytree2021,poliTreeSofware}, uniform mass binning \citep{unifMassBinning2001,pmlr-v139-gupta21b}, or other space partitioning approach of interest. So long as the outcomes for prediction, $(Y_i;\ i\in\scal)$, stay independent of the learned partition split points conditionally on their treatment designation and co-variates, $((W_i,X_i);\ i\in\scal)$, our result is technically ``model-free'' among possible approaches to discretize $\xcal$. We illustrate an example partition of the feature space using a policy tree \citep{policytree2021,poliTreeSofware} in Section \ref{sec:empiricalEval}.

Of course, the exact partitioning approach will contribute to whether the treatment effect parameters are informative or not. Compared to a partition of the feature space that results from a least squares criterion for a single tree in a regression forest that predicts $Y_i$ with $X_i$, it may be more ideal to understand treatment effects with the feature space partition given by a decision tree in a causal forest \citep{grf2019} that uses the more relevant Robinson-Learner objective \citep{nieWager2020}.



\subsection{Given a fixed sample size\label{sec:fixedSampleSize}}
Let \[
n_\scal = \abs{\scal}
\] 
be the fixed sample size corresponding to the honest prediction set $\scal$. Recall that
\[
\sum_{w,l}n_{wl} = n_\scal.
\]
For the sake of solving for tolerable values of $K,L,\alpha,$ and $\epsilon$, let us set:
\[
\min_{w,l} n_{wl} \dot{=} \frac{n_\scal}{K\times L}.
\]
To apply Propositions \ref{prop:CLTBased}, \ref{prop:boundedOutcome}, and \ref{prop:boundedOutcomeBoundedVar} to the case with a fixed sample size budget $n_\scal$, we can enforce an inequality of the form
\[
\frac{n_\scal}{K\times L}\ \geq\ g(K,L,\alpha,\epsilon,\sigma^2,\dots),
\]
where $g$ corresponds to the concentration inequality that is used to determine the sufficient sample size, such as the right hand side of the inequality in \eqref{eqn:cltIneq}. In order to solve for a tolerable value for one of $K$, $L$, $\alpha$, and $\epsilon$, we can fix the other terms. That is:
\begin{itemize}
\item{The maximum allowable number of treatment arms:}
\[
\arg\max\curl{K:\ \frac{n_\scal}{K\times L}\ \geq\ g(K,L,\alpha,\epsilon,\sigma^2,\dots)}.
\]
\item{The maximum allowable number of feature space subsets:}
\[
\arg\max\curl{L:\ \frac{n_\scal}{K\times L}\ \geq\ g(K,L,\alpha,\epsilon,\sigma^2,\dots)}.
\]
\item{The supremum allowable confidence:}
\[
\arg\sup\curl{1-\alpha:\ \frac{n_\scal}{K\times L}\ \geq\ g(K,L,\alpha,\epsilon,\sigma^2,\dots)}.
\]
\item{The infimum allowable margin of error:}
\[
\arg\inf\curl{\epsilon:\ \frac{n_\scal}{K\times L}\ \geq\ g(K,L,\alpha,\epsilon,\sigma^2,\dots)}.
\]
\end{itemize}

If the solutions to these are not analytically derivable, such as $\epsilon$ when $g$ is given by Bennet's inequality in Proposition \ref{prop:boundedOutcomeBoundedVar}, bisection or some other root-finding method can be used \citep{Rosenblum2009-ue_concInequalities}.

\subsection{Bounding the Outcome's conditional variance\label{sec:boundCondVar} }
Some warnings are in order, especially when attempting to understand the maximum allowable feature space subsets or maximum allowable number of treatment arms. In particular, special care must be taken with respect to the outcome variation bound
\[
\sigma^2\geq\max_{w,l}\var{\ynp\mid \wnp=w,\ell(\xnp)=l}.
\]
Due to the possibility of heteroscedasticty, $\var{\ynp\mid \wnp=w,\ell(\xnp)=l}$ may change non-trivially as $L$ and $K$ change. 
When the outcome is bounded almost surely--Assumption \ref{assump:boundedOutcome}, it can be shown:
\[
\max_{w,l}\var{\ynp\mid \wnp=w,\ell(\xnp)=l}\leq\frac{(b-a)^2}{4}.
\]
The upper bound is attained if $\ynp$ in treatment group $w$ in leaf $l$ is discrete at two unique values $a$ and $b$ with equal probability. 


A scenario to tighten the conservative conditional variance bound of $(b-a)^2/4$ can be if we have knowledge that, for some fixed outcome value $y^*$ and some small $p^*\in\sqbrack{0,1/2}$, we have:
\[
\max_{w,l}\prc{\ynp\neq y^*}{\wnp=w,\ell(\xnp)=l}\ \leq p^*.
\]
The rate $p^*$ can be approximated from domain knowledge in the case that an outcome rarely exhibits a change from a fixed value. For example, the best click-through rate among marketing ad interventions may realistically remain capped by a small value. In this case, we show in Appendix \ref{append:proofs} that:
\begin{equation}\begin{aligned}\label{eqn:nonZeroVarBound}
&\max_{w,l}\var{\ynp\mid \wnp=w,\ell(\xnp)=l}
\\\leq\  &p^*\frac{(b-a)^2}{4}+p^*(1-p^*)\max\curl{(a-y^*)^2,(b-y^*)^2}.
\end{aligned}\end{equation}

If $\ynp\in\{0,1\}$, the conditional variance upper bound can instead be tightened by noting:
\[\begin{aligned}
&\max_{w,l}\var{\ynp\mid \wnp=w,\ell(\xnp)=l }\\
=\ &\max_{w,l}\{\prc{\ynp=y^*}{ \wnp=w,\ell(\xnp)=l }\\
&\ \ \ \ \times \sqbrack{1-\prc{\ynp=y^*}{ \wnp=w,\ell(\xnp)=l }}\}\\
\leq\ &p^*(1-p^*),
\end{aligned}\]
since $g(t)=t(1-t)$ is monotone increasing for $t\in[0,1/2]$. Figure \ref{fig:sufficient} depicts the binary outcome case with the most conservative standard deviation, $\sigma=0.5$, corresponding to $p^*=0.5$.



Should we instead wish to specify the parameter $\sigma=1$, the sufficient sample size in Propositions \ref{prop:CLTBased}, \ref{prop:boundedOutcome}, and \ref{prop:boundedOutcomeBoundedVar} guarantees a deviation that is within a factor $\epsilon$ of the conditional standard deviation, as we explain in Remark \ref{rem:stdScale}.

\section{Emprical Evaluation\label{sec:empiricalEval}}

\begin{figure}
\begin{minipage}[b]{.475\textwidth}
\centering
\includegraphics[width=1\textwidth]{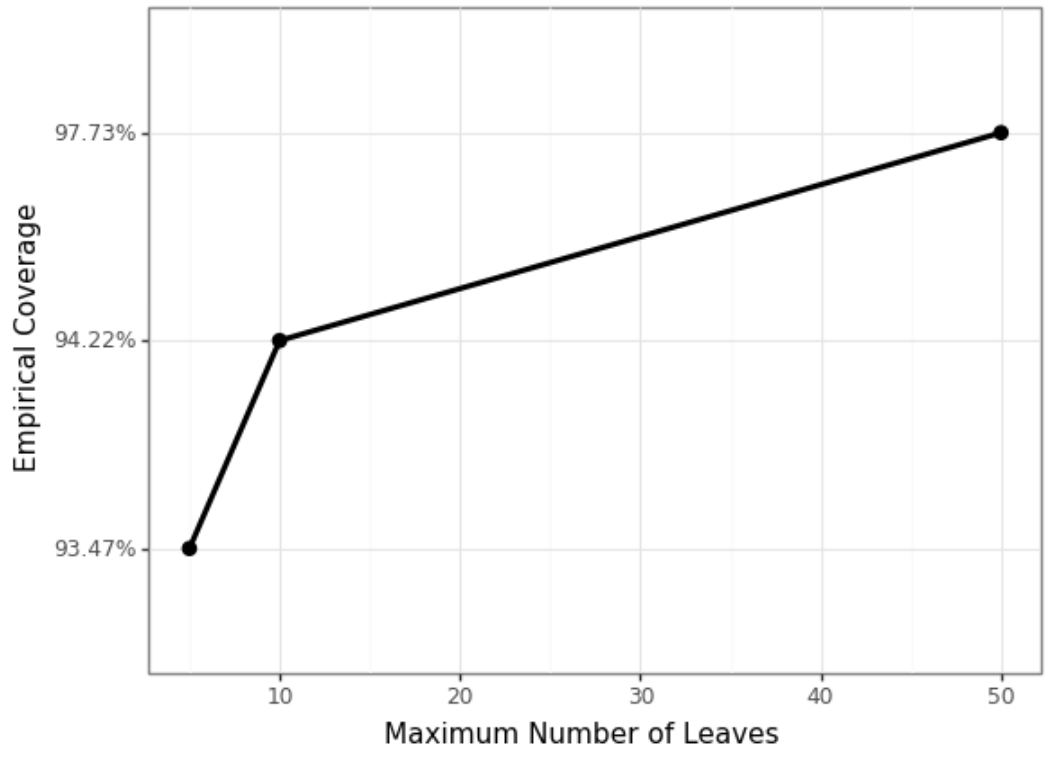}
\caption{Coverage across 500 replicates.\label{fig:simsEmpiricalCoverage}}
\end{minipage}
\hfill
\begin{minipage}[b]{.475\textwidth}
\centering
\includegraphics[width=1\textwidth]{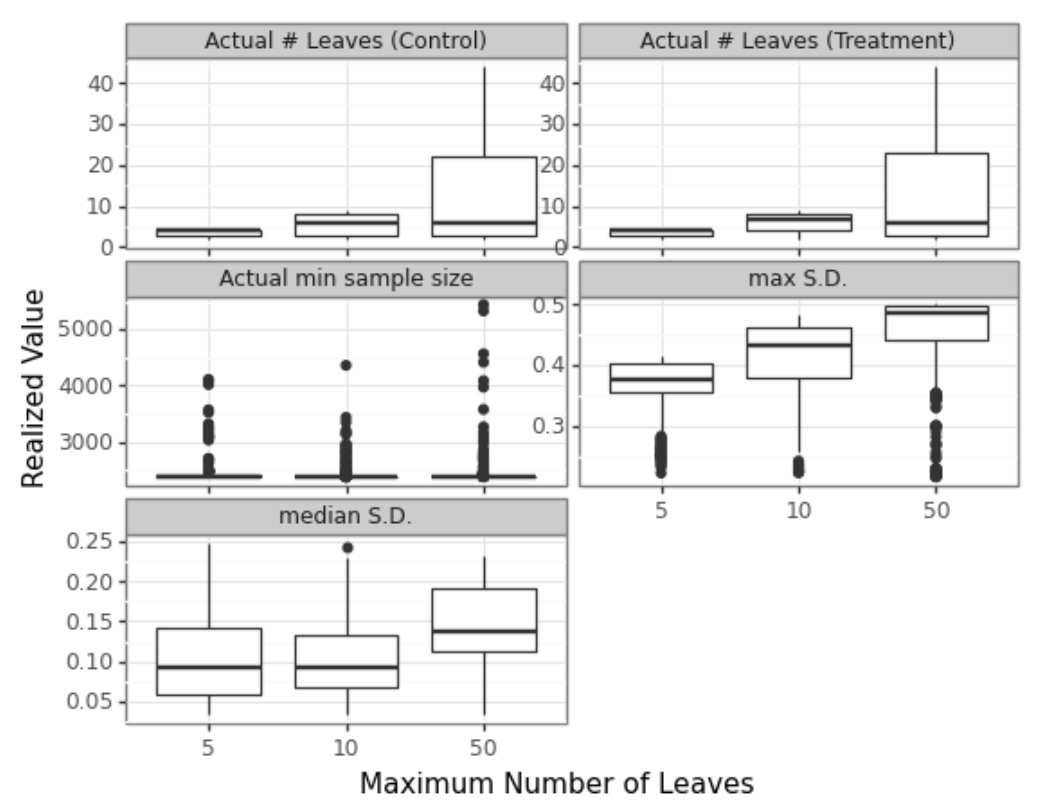}
\caption{Across 500 replicates, the realized decision tree parameters.\label{fig:simsRealizedParams}}
\end{minipage}
\end{figure}





We now present an application to Criteo AI Lab's uplift prediction benchmark dataset \citep{diemert:hal-02515860}. According to the web page that hosts the data:
\begin{quote}
This dataset is constructed by assembling data resulting from several incrementality tests, a particular randomized trial procedure where a random part of the population is prevented from being targeted by advertising. It consists of 25M rows, each one representing a user with [12] features, a treatment indicator and 2 labels (visits and conversions).
\end{quote}

The available down sampled data consists of 13,979,592 observations. The outcome we focus on is the $0/1$ indicator for a website visit.

In our semi-synthetic simulations, we use the Central Limit Theorem approach to designate a sample size that guarantees
\begin{equation}\begin{aligned}\label{eqn:guaranteeNeeded}
&\pr{ \max_w\abs{\frac{ \hat{\mu}(\xnp,w)-\mu(\xnp,w) }{\sigma_{w\ell(\xnp)} }}<\frac{1}{25}  }\\
\end{aligned}\end{equation}
be at least $0.9$. For this to hold, it is sufficient that
\[
\min_{w,l}n_{wl}\geq 2374.
\]

We implement the sub-group counterfactual mean estimation using a \texttt{PolicyTree} \citep{policytree2021} from the \texttt{econml.policy} module in Python \citep{econml}, which learns a decision tree-based mapping, also known as a treatment policy,
\[
\hat{e}:\ \xcal\to\curl{1,\dots,L}\to\curl{1,\dots,K}
\]
maximizing the total rewards (website visits):
\[
\mathcal{L}(e)\ \dot{=}\ \sum_{i\in\scal}\sum_w \indic{{e}(X_i)=w}R_{iw}.
\]
For our case, we constrain the minimum number of samples in each leaf to be \texttt{min\_samples\_leaf=2374}. Here, $R_{iw}$ is the reward that can be obtained from record $i$ should we assign them to treatment group $w$. In practice, the various potential rewards (potential outcomes) are unknown, so we can impute each of these with our estimate of
\[
\EC{Y_i}{W_i=w,X_i}.
\]

Using the training set indices $\curl{1,\dots,n}\backslash\scal$, we fit a single regression function, or Single-learner as it has been referred to in the causal inference literature \citep{kunzelXLearner2019,nieWager2020}, using the histogram gradient boosting classifier in the \texttt{sklearn.ensemble} module by setting \texttt{min\_samples\_leaf=2374} \citep{scikit-learn}. After imputing the honest set values of $R_{iw}$ using this S-learner, we set the argument \texttt{honest} to be \texttt{False} in \texttt{PolicyTree} since our honest splitting was done manually; we leave the rest of the arguments to their default. Because the \texttt{PolicyTree} module does not allow us to constrain the treatment-group sample size per leaf, we enforce our required leaf constraint per treatment group by learning one \texttt{PolicyTree} separately for each treatment group.

Given a specified maximum number of leaves ($L=5,10,50$) in a \texttt{PolicyTree}, and given the minimum sample size per treatment group per leaf of 2374, along with an honest set fraction of $\abs{\scal}/{n}=0.5$, we sample (with replacement)
\[
n = 2\times L\times 2374/0.5
\]
of the 13,979,592 records (half for the control group and half for treatment group). We train our model with this sample. We replicate this process 500 times.

To approximate the true parameters in \eqref{eqn:guaranteeNeeded}, we sample in each of the 500 replicates an additional $n_{\rm test}=2\times L\times 20,000$ records with which we calculate
\[
\hat{\mu}^{\rm test}(x,w):=\frac{1}{ \abs{ \ical^{\rm test}(w,\ell(x) } }\sum_{ j\in\ical^{\rm test}(w,\ell(x) } Y_{n+j}
\]
and
\[
\hat{\sigma}^{\rm test}_{w\ell(x)}:=\sqrt{\frac{1}{ \abs{ \ical^{\rm test}(w,\ell(x)) } }\sum_{ j\in\ical^{\rm test}(w,\ell(x)) } \parenth{Y_{n+j}-\hat{\mu}^{\rm test}(x,w) }^2 }.
\]
Here,
\[
\ical^{\rm test}(w,l):=\curl{j\in\curl{1,\dots,n_{\rm test}} :\ W_{n+j}=w,\ell( X_{n+j} )=l }.
\]
To test whether \eqref{eqn:guaranteeNeeded} holds, we take the average across the 500 replicates of the one-replicate empirical coverage at the $n_{test}$ test points:
\[
\frac{1}{n_{\rm test}}\sum_{j=1}^{n_{\rm test} }\indic{ \max_w\abs{\frac{ \hat{\mu}(X_{n+j},w)-\hat{\mu}^{\rm test}(X_{n+j},w) }{\hat{\sigma}^{\rm test}_{w\ell(X_{n+j} )} }}<\frac{1}{25}  }.
\]

Figure \ref{fig:simsEmpiricalCoverage} shows the results: we achieve greater than 90\% coverage across the replicates. Besides the minimum sample size per leaf and treatment group we specify, Figure \ref{fig:simsRealizedParams} explains our good performance by demonstrating the distribution of realized decision tree parameters:  number of leaves, minimum sample size per leaf and treatment group, and the max and median standard deviations per leaf and treatment group in the test set. For example, the number of leaves in a decision tree was not always the maximum of $L$ allowed by the sample size $n$ and lower bound on $\min_{w,l}n_{wl}$. Intuitively, this is because better splits were possible at the top of the decision trees compared to a split that evenly divided the records in $\scal$ to the left and right child nodes.


\section{Discussion and Limitations\label{sec:discussion}}

With regard to recent advances in online experimentation and their anytime valid inference \citep{Maharaj_2023,kuchibhotla2021nearoptimal}, we think our results are helpful despite their application to fitting a model "offline" after data collection. Online experiments typically run until statistically significant differences across pair-wise mean comparisons are reached, or until some time budget is reached if no statistical significance. For the first case, at the time point where statistical significance is reached, we can see how complex a model can be learned using the discussion in Section \ref{sec:fixedSampleSize}. If the complexity is not satisfactory, more data can be collected. On the other hand, despite the possibility that pair-wise mean comparisons are not marginally significantly different when a time budget is reached, some pair-wise comparisons in learnable sub-groups may be significant if a large enough signal can be denoised by dilineating sub-groups with pre-treatment features. For this second case, we may similarly entertain plausible model complexities and even attempt to collect further data until a desired complexity is possible. 

Regarding the determination of a statistically significant difference between two treatment groups in a sub-group, the value $\epsilon$ in our sample size calculation gives an understanding that an absolute difference in means of $2\epsilon$ or larger, whether at the original scale or a standardized scale, is sufficiently large to rule out random chance at the $\alpha$-significance level. However, we note that our use of an upper bound on the conditional variances opens up the possiblity that smaller differences than $2\epsilon$ can be deemed statistically significant in practice. This is because the standard error for the mean comparison, 
\[
\sqrt{\var{\hat{\mu}_{wl}-\hat{\mu}_{w'l}\mid \ical }}=\sqrt{\sigma_{wl}^2/n_{wl}+\sigma_{w'l}^2/n_{w'l}}, 
\]
can be much smaller when using the true conditional variances compared to the maximum bound $\sigma^2$. Related to this, each treatment group need not have the same number of units either. That is, an imbalanced design where there are more units in some treatment groups compared to others could help make smaller differences in a pair-wise comparison of means detectable, espcially if a greater sample is directed toward the treatment group(s) with more outcome variation. This is simply because the standard error for a pairwise mean comparison also decreases with an increase in the sample size used to calculate one of the two sample means, $\hat{\mu}_{wl}$ or $\hat{\mu}_{w'l}$, in question. 

While the empirical counterfactual mean in a sub-group may be close, in a frequentist sense, to its population analogue with a sufficient sample size, special care must be taken with respect to this parameter. Firstly, the sub-group counterfactual average may not be the best summary value for a treatment's performance: a random variable (e.g. difference in potential outcomes) need not be close to its (sub-group) average \citep{ruizPadilla2022Pibt}. Relatedly, special care must be taken with respect to hyper-parameter tuning and goodness of fit. Empirics when choosing a regression model, such as $R^2$ and AUC, are still important. Moreover, an ensemble of trees may more robustly summarize an experiment compared to one tree \citep{grf2019}. 


In this paper, we address only internal validity of our inference (recall the stationarity condition of the new test point in Section \ref{sec:manageableInfTarget}). In practice, sensitivity analysis is important when reasoning about external validity of inference with an experimental sample. For example, see \citet{lei2023policy} for a discussion of this phenomenon along with a reformulation of the policy learning objective to account for it. Moreover, with respect to designing randomized trials with the aim of understanding treatment effect heterogeneity, \citet{tipton2023designing} study the mean squared prediction error (MSPE) of the \textit{random} individual treatment effect of an out of sample subject in a target population. Heterogeneity in \citet{tipton2023designing} is with respect to a linear specification and a low-dimensional reweighted least squares estimator accounting for how relatively likely we are to observe a unit in the sample's origin population compared to the target inference population. It would be interesting to join our and this approach for external validity under a nonlinear model specification including, but not limited to, a learnable feature space partition.

\section{Acknowledgements}
Gabriel would like to acknowledge Binjie Lai, Julia Viladomat, Vasco Yasenov, Oscar Madrid-Padilla, Alan Vasquez, David Arbour, Ritwik Sinha, and Avi Feller for their insightful conversations related to this manuscript. Gabriel would also like to acknowledge the financial support of Adobe Inc. during the realization of this work.

\bibliographystyle{apalike}  
\bibliography{references}  
\appendix

\newpage

\section{Extended Main Results\label{append:extendedResults}}
\begin{proposition}[Main Result: Bounded Outcome\label{prop:boundedOutcome}]\ \\

Suppose Assumption \ref{assump:boundedOutcome} holds. Let $\alpha\in(0,1)$ and $\epsilon>0$.
\begin{itemize}
\item Then for \eqref{eqn:accurateCondMeanRandomX}
to hold, it is sufficient that
\[
\min_{ w,l }n_{ wl }\geq \frac{ \log\parenth{ \frac{2}{ 1-\parenth{1-\alpha }^{ \frac{1}{K} } } }\abs{b-a}^2 }{2\epsilon^2}.
\]
\item Then for \eqref{eqn:accurateCondMeanSupX}
to hold, it is sufficient that
\[
\min_{ w,l }n_{ wl }\geq \frac{ \log\parenth{ \frac{2}{ 1-\parenth{1-\alpha }^{ \frac{1}{KL} } } }\abs{b-a}^2 }{2\epsilon^2}.
\]
\end{itemize}
\end{proposition}
\begin{proposition}[Main Result: Bounded Outcome and Bound on its Conditional Variance\label{prop:boundedOutcomeBoundedVar}]\ \\

Suppose Assumptions \ref{assump:boundedOutcome} and \ref{assump:boundedVariance} hold. 

\begin{itemize}
\item Then for the guarantee in \eqref{eqn:accurateCondMeanRandomX} to hold, it is sufficient that
\[\begin{aligned}
&\min_{ w,l } n_{ wl }&\geq\ & \frac{ \log\parenth{ \frac{2}{ 1-(1-\alpha)^{\frac{1}{K}} } }\max\{a^2,b^2\} }{ \sigma^2 \curl{ \parenth{1+s }\log  \parenth{1+s } -s } }.\\
\end{aligned}\]

\item Then for the guarantee in \eqref{eqn:accurateCondMeanSupX} to hold, it is sufficient that
\[\begin{aligned}
&\min_{ w,l } n_{ wl }&\geq\ & \frac{ \log\parenth{ \frac{2}{ 1-(1-\alpha)^{\frac{1}{KL}} } }\max\{a^2,b^2\} }{ \sigma^2 \curl{ \parenth{1+s }\log  \parenth{1+s } -s } }.\\
\end{aligned}\]

Here,
\[
s := \frac{ \epsilon\max\{|a|,|b|\} }{\sigma^2}.
\]

\end{itemize}
\end{proposition}

\section{Proofs\label{append:proofs}}

\subsection{Proof of Lemma \ref{lem:keyIneqs}}
\begin{proof}[Proof of Lemma \ref{lem:keyIneqs}]\ \\

\textbf{Guarantees For a Random $\xnp\indep\curl{ X_1,\dots,X_n}$:}\\
With respect to the part of Lemma \ref{lem:keyIneqs} that provides a guarantee for random $\xnp$, consider the event:
\[\begin{aligned}
&\ecal(\xnp)&=\ &\curl{ \max_w \abs{ \hat{\mu}(\xnp,w)-\mu(\xnp,w) }<\epsilon  }.\\
\end{aligned}\]

Recall our assumption that $\prc{\ecal(\xnp)}{\icaln}\geq 1-\alpha$.

\begin{itemize}
\item\textbf{Best Conditional Counterfactual Mean:}\\
Consider that:
\[\begin{aligned}
&\prc{\abs{ \max_w\hat{\mu}(\xnp,w)-\max_w\mu(\xnp,w)  }\geq \epsilon}{\icaln}\\
=\ &\prc{\abs{ \max_w\hat{\mu}(\xnp,w)-\max_w\mu(\xnp,w) }\geq \epsilon}{\ecal(\xnp),\icaln}\prc{\ecal(\xnp)}{\icaln}\\
&+\prc{\abs{ \max_w\hat{\mu}(\xnp,w)-\max_w\mu(\xnp,w) }\geq \epsilon}{\ecal(\xnp)^C,\icaln}\prc{\ecal(\xnp)^C}{\icaln}\\
\leq\ &\prc{\ecal(\xnp)^C}{\icaln}\\
\leq\ &\alpha.
\end{aligned}\]

From this, it follows that
\[
\prc{\abs{ \max_w\hat{\mu}(\xnp,w)-\max_w\mu(\xnp,w)  }<\epsilon}{\icaln}\geq 1-\alpha
\]
as desired.

We used that
\[
\prc{\abs{ \max_w\hat{\mu}(\xnp,w)-\max_w\mu(\xnp,w) }\geq \epsilon}{ \ecal(\xnp),\icaln}=0.
\]
This is because $\ecal(\xnp)$ and triangle inequality imply that with probability $1$,
\[\begin{aligned}
&\abs{ \max_w\hat{\mu}(\xnp,w)-\max_w\mu(\xnp,w) }\\
=\ &\max\curl{ \max_w\hat{\mu}(\xnp,w)-\max_w\mu(\xnp,w),\ \max_w\mu(\xnp,w)-\max_w\hat{\mu}(\xnp,w) }\\
<\ &\max\curl{ \max_w\sqbrack{\mu(\xnp,w)+\epsilon}-\max_w\mu(\xnp,w),\ \max_w\sqbrack{\hat{\mu}(\xnp,w)+\epsilon}-\max_w\hat{\mu}(\xnp,w) }\\
=\ &\epsilon.
\end{aligned}\]
Moreover, we used that 
\[
\prc{\abs{ \max_w\hat{\mu}(\xnp,w)-\max_w\mu(\xnp,w) }\geq \epsilon}{\ecal(\xnp)^C,\icaln}\leq 1.
\]

\item\textbf{Treatment Effect Function:} \\
Consider next that:
\[\begin{aligned}
&\prc{\max_{ w\neq w' }\abs{ \hat{\tau}( \xnp,w,w' )-\tau(\xnp,w,w') }\geq 2\epsilon}{\icaln}\\
=\ &\prc{\max_{ w\neq w' }\abs{ \hat{\tau}( \xnp,w,w' )-\tau(\xnp,w,w') }\geq 2\epsilon}{\ecal(\xnp),\icaln}\prc{\ecal(\xnp)}{\icaln}\\
&+\prc{\max_{ w\neq w' }\abs{ \hat{\tau}( \xnp,w,w' )-\tau(\xnp,w,w') }\geq 2\epsilon}{\ecal(\xnp)^C,\icaln}\prc{\ecal(\xnp)^C}{\icaln}\\
\leq\ &\prc{\ecal(\xnp)^C}{\icaln}\\
\leq\ &\alpha.
\end{aligned}\]

From this, it follows that
\[
\prc{\max_{ w\neq w' }\abs{ \hat{\tau}( \xnp,w,w' )-\tau(\xnp,w,w') }< 2\epsilon}{\icaln}\geq 1-\alpha
\]
as desired.

We used that
\[
\prc{ \max_{ w\neq w' }\abs{ \hat{\tau}( \xnp,w,w' )-\tau(\xnp,w,w') }\geq 2\epsilon}{ \ecal(\xnp),\icaln}=0.
\]
This is because $\ecal(\xnp)$ and triangle inequality imply that with probability $1$,
\[
\max_{ w\neq w' }\abs{ \hat{\tau}( \xnp,w,w' )-\tau(\xnp,w,w') }\leq 2 \max_{w}\abs{ \hat{\mu}(\xnp,w)-\mu(\xnp,w) }<2\epsilon.
\]
Moreover, we used that 
\[
\prc{\max_{ w\neq w' }\abs{ \hat{\tau}( \xnp,w,w' )-\tau(\xnp,w,w') }\geq 2\epsilon}{\ecal(\xnp)^C,\icaln}\leq 1.
\]

\end{itemize}


\textbf{Guarantees Uniformly Across $x\in\xcal$:}\\
With respect to the part of Lemma \ref{lem:keyIneqs} that provides a guarantee uniformly across $x$, consider the event:
\[\begin{aligned}
&\ecal&=\ &\curl{ \sup_x\max_{w}\abs{ \hat{\mu}(x,w)-\mu(x,w) }<\epsilon  }.\\
\end{aligned}\]

Recall our assumption that $\prc{\ecal}{\icaln}\geq 1-\alpha$.

\begin{itemize}
\item\textbf{Best Conditional Counterfactual Mean:}\\
Consider that:
\[\begin{aligned}
&\prc{\sup_x\abs{ \max_w\hat{\mu}(x,w)-\max_w\mu(x,w)  }\geq \epsilon}{\icaln}\\
=\ &\prc{\sup_x\abs{ \max_w\hat{\mu}(x,w)-\max_w\mu(x,w) }\geq \epsilon}{\ecal,\icaln}\prc{\ecal}{\icaln}\\
&+\prc{\sup_x \abs{ \max_w\hat{\mu}(x,w)-\max_w\mu(x,w) }\geq \epsilon}{\ecal^C,\icaln}\prc{\ecal^C}{\icaln}\\
\leq\ &\prc{\ecal^C}{\icaln}\\
\leq\ &\alpha.
\end{aligned}\]

From this, it follows that
\[
\prc{\sup_x\abs{ \max_w\hat{\mu}(x,w)-\max_w\mu(x,w)  }<\epsilon}{\icaln}\geq 1-\alpha
\]
as desired.

We used that
\[
\prc{\sup_x\abs{ \max_w\hat{\mu}(x,w)-\max_w\mu(x,w) }\geq \epsilon}{ \ecal,\icaln}=0.
\]
This is because $\ecal$ and triangle inequality imply that with probability $1$,
\[\begin{aligned}
&\sup_x\abs{ \max_w\hat{\mu}(x,w)-\max_w\mu(x,w) }\\
=\ &\sup_x\max\curl{ \max_w\hat{\mu}(x,w)-\max_w\mu(x,w),\ \max_w\mu(x,w)-\max_w\hat{\mu}(x,w) }\\
<\ &\sup_x\max\curl{ \max_w\sqbrack{\mu(x,w)+\epsilon}-\max_w\mu(x,w),\ \max_w\sqbrack{\hat{\mu}(x,w)+\epsilon}-\max_w\hat{\mu}(x,w) }\\
=\ &\epsilon.
\end{aligned}\]
Moreover, we used that 
\[
\prc{\sup_x \abs{ \max_w\hat{\mu}(x,w)-\max_w\mu(x,w) }\geq \epsilon}{\ecal^C,\icaln}\leq 1.
\]

\item\textbf{Treatment Effect Function:} \\
Consider next that:
\[\begin{aligned}
&\prc{\sup_x\max_{ w\neq w' }\abs{ \hat{\tau}( x,w,w' )-\tau(x,w,w') }\geq 2\epsilon}{\icaln}\\
=\ &\prc{\sup_x\max_{ w\neq w' }\abs{ \hat{\tau}( x,w,w' )-\tau(x,w,w') }\geq 2\epsilon}{\ecal,\icaln}\prc{\ecal}{\icaln}\\
&+\prc{\sup_x\max_{ w\neq w' }\abs{ \hat{\tau}( x,w,w' )-\tau(x,w,w') }\geq 2\epsilon}{\ecal^C,\icaln}\prc{\ecal^C}{\icaln}\\
\leq\ &\prc{\ecal^C}{\icaln}\\
\leq\ &\alpha.
\end{aligned}\]
From this, it follows that
\[
\prc{\sup_x\max_{ w\neq w' }\abs{ \hat{\tau}( x,w,w' )-\tau(x,w,w') }< 2\epsilon}{\icaln}\geq 1-\alpha
\]
as desired.
We used that
\[
\prc{\sup_x\max_{ w\neq w' }\abs{ \hat{\tau}( x,w,w' )-\tau(x,w,w') }\geq 2\epsilon}{ \ecal,\icaln}=0.
\]
This is because $\ecal$ and triangle inequality imply that with probability $1$,
\[
\sup_x\max_{ w\neq w' }\abs{ \hat{\tau}( x,w,w' )-\tau(x,w,w') }\leq 2\sup_x \max_{w}\abs{ \hat{\mu}(x,w)-\mu(x,w) }<2\epsilon.
\]
Moreover, we used that 
\[
\prc{\sup_x\max_{ w\neq w' }\abs{ \hat{\tau}( x,w,w' )-\tau(x,w,w') }\geq 2\epsilon}{\ecal^C,\icaln}\leq 1.
\]

\end{itemize}

\end{proof}

\subsection{Proof of Propositions \ref{prop:CLTBased}, \ref{prop:boundedOutcome}, and \ref{prop:boundedOutcomeBoundedVar}\label{append:proofMainResults}}

Consider first that we have $M=K\times L$ groups we are taking the mean of. Each observation is independent of any others, and within each group defined by the treatment group and feature space subset, they are iid according to 
\[
Y_i|\ell(X_i)=l,W_i=w; l=1,\dots,L,\ w=1,\dots,K.
\]
Thus, we can apply the pertinent result in Lemma \ref{lem:jointEstMeansClusters}.

\subsection{Proof of Equation \eqref{eqn:nonZeroVarBound}}
\begin{proof}[Proof of Equation \eqref{eqn:nonZeroVarBound}]\ \\
We have:

\[\begin{aligned}
&\var{\ynp\mid \wnp=w,\ell(\xnp)=l}\\
=\ &\var{\ynp-y^*\mid \wnp=w,\ell(\xnp)=l}\\
=\ &\prc{\ynp\neq y^*}{\wnp=w,\ell(\xnp)=l}\EC{(\ynp-y^*)^2}{\wnp=w,\ell(\xnp)=l,\ynp\neq y^*}\\
&-\parenth{\prc{\ynp\neq y^*}{\wnp=w,\ell(\xnp)=l}\EC{\ynp-y^*}{\wnp=w,\ell(\xnp)=l,\ynp\neq y^*}}^2\\
=\ &\prc{\ynp\neq y^*}{\wnp=w,\ell(\xnp)=l}\var{\ynp-y^*\mid \wnp=w,\ell(\xnp)=l,\ynp\neq y^*}\\
&+\prc{\ynp\neq y^*}{\wnp=w,\ell(\xnp)=l}\sqbrack{1-\prc{\ynp\neq y^*}{\wnp=w,\ell(\xnp)=l}}\\
&\times\parenth{\EC{\ynp-y^*}{\wnp=w,\ell(\xnp)=l,\ynp\neq y^*}}^2\\
\leq\ &p^*\frac{(b-a)^2}{4}+p^*(1-p^*)\max\curl{(a-y^*)^2,(b-y^*)^2}
\end{aligned}\]

In the first equality, we use that variance does not change when adding a constant. In the second equality, we use that $\var{U}=\E{U^2}-(\E{U})^2$, along with law of total expectation for these two terms. The third equality uses that $\E{U^2}=\var{U}+(\E{U})^2$ and combines terms. The inequality at the end uses that the function $g(t)=t(1-t)$ is monotone increasing on the interval $t\in[0,1/2]$. We also use the assumption that
\[
\max_{w,l}\prc{\ynp\neq y^*}{\wnp=w,\ell(\xnp)=l}\leq p^*\leq \frac{1}{2},
\]
along with the fact that $\ynp-y^*\in[a-y^*,b-y^*]$.

\end{proof}
\subsection{Additional Lemmas}

\begin{lemma}[Joint estimation of independent means\label{lem:jointEstMeans}]\ \\

Consider a sample of $n_j$ iid random variables $Y_{1j},\dots,Y_{n_jj}$, across $j=1,\dots,K$. 

Let $\epsilon\in(0,1)$ and $\alpha\in(0,1)$. Denote
\[
\mu_j= \E{Y_{ij}}.
\] 
\begin{itemize}
\item Suppose that 
\[
\pr{a\leq Y_{ij}\leq b }=1.
\] Then for
\begin{equation}\label{eqn:jointMeanEst}
\pr{ \max_{j}\abs{ \frac{1}{n_j}\sum_{i=1}^{n_j}Y_{ij}-\mu_j }<\epsilon }\geq1-\alpha,
\end{equation}
it is sufficient that
\[
\min_j n_j\geq \frac{ \log\parenth{ \frac{2}{ 1-(1-\alpha)^{\frac{1}{K}} } }\abs{b-a}^2 }{2\epsilon^2}.
\]

\item In addition to
\[
\pr{a\leq Y_{ij}\leq b }=1,
\]
suppose further that
\[
\var{Y_{ij} }\leq \sigma^2\ \text{for each }j=1,\dots,K.
\]
Then for Equation \eqref{eqn:jointMeanEst}, it is sufficient that
\[\begin{aligned}
&\min_j n_j&\geq\ & \frac{ \log\parenth{ \frac{2}{ 1-(1-\alpha)^{\frac{1}{K}} } }\max\{a^2,b^2\} }{ \sigma^2 \curl{ \parenth{1+s }\log  \parenth{1+s } -s } },\\
\end{aligned}\]

where \[
s = \frac{ \epsilon\max\{|a|,|b|\} }{\sigma^2}.
\]

\item Suppose instead that the Central Limit Theorem (CLT) holds:
\[
\hat{\mu}_{j}\sim\mathcal{N}\parenth{\mu_j,\frac{\var{Y_{ij}}}{n_j} };\ j=1,\dots,K.
\]
Then for Equation \eqref{eqn:jointMeanEst}, it is sufficient that
\[
\min_j n_j\ \geq\ \parenth{\frac{z_{\alpha_0/2}\times \max_j\sigma_j}{\epsilon}}^2.
\]
Here, $\Phi$ denotes the standard Normal CDF, while $z_{\alpha_0/2}$ is the $(1-\alpha_0/2)^{\rm th}$ quantile of the standard Normal distribution. Meanwhile,
\[
\alpha_0\ \dot{=}\ 1-(1-\alpha)^{\frac{1}{K}}, 
\]

\end{itemize}
\begin{proof}\ \\
\begin{itemize}
    \item For the first case, we have
\[\begin{aligned}
&\pr{ \max_{j}\abs{ \frac{1}{n_j}\sum_{i=1}^{n_j}Y_{ij}-\mu_j }<\epsilon }&\\
=\ & \prod_{j=1}^K \pr{ \abs{ \frac{1}{n_j}\sum_{i=1}^{n_j}Y_{ij}-\mu_j }<\epsilon }\\
\geq\ & \prod_{j=1}^K \parenth{ 1-\exp\curl{-\frac{2n_j\epsilon^2}{ \abs{b-a}^2 } } } \\
\geq\ &1- \alpha. \\
\end{aligned}\]
The first equality is due to independence. The next inequality is due to Hoeffding's inequality \citep{wainwright_2019}. For the last inequality, it is sufficient that
\[
\min_j n_j\geq \frac{ \log\parenth{ \frac{2}{ 1-(1-\alpha)^{\frac{1}{K}} } }\abs{b-a}^2 }{2\epsilon^2}.
\]

\item For the second case, we use Bennett's inequality \citep{bennet1962,Rosenblum2009-ue_concInequalities} with the premise that $\max_j\var{Y_{ij}}\leq\sigma^2$. We get:
\[\begin{aligned}
&\pr{ \max_{j}\abs{ \frac{1}{n_j}\sum_{i=1}^{n_j}Y_{ij}-\mu_j }<\epsilon }&\\
=\ & \prod_{j=1}^K \pr{ \abs{ \frac{1}{n_j}\sum_{i=1}^{n_j}Y_{ij}-\mu_j }<\epsilon }\\
\geq\ & \prod_{j=1}^K \parenth{ 1-\exp\curl{-\frac{n_j \sigma^2 }{ \max\{a^2,b^2\} }\curl{ \parenth{1+\frac{ \epsilon\max\{|a|,|b|\} }{\sigma^2} }\log  \parenth{1+\frac{ \epsilon\max\{|a|,|b|\} }{\sigma^2} } -\frac{ \epsilon\max\{|a|,|b|\} }{\sigma^2} } }  } \\
\geq\ &1- \alpha. \\
\end{aligned}\]

For the last inequality it is sufficient that,
\[\begin{aligned}
&\min_j n_j&\geq\ & \frac{ \log\parenth{ \frac{2}{ 1-(1-\alpha)^{\frac{1}{K}} } }\max\{a^2,b^2\} }{ \sigma^2 \curl{ \parenth{1+\frac{ \epsilon\max\{|a|,|b|\} }{\sigma^2} }\log  \parenth{1+\frac{ \epsilon\max\{|a|,|b|\} }{\sigma^2} } -\frac{ \epsilon\max\{|a|,|b|\} }{\sigma^2} } }\\
& &=\ & \frac{ \log\parenth{ \frac{2}{ 1-(1-\alpha)^{\frac{1}{K}} } }\max\{a^2,b^2\} }{ \sigma^2 \curl{ \parenth{1+s }\log  \parenth{1+s } -s } },\\
\end{aligned}\]

where \[
s = \frac{ \epsilon\max\{|a|,|b|\} }{\sigma^2}.
\]

\item In the third case, we have:
\[\begin{aligned}
&\pr{\max_j\abs{\hat{\mu}_j-\mu_j}<\epsilon}&\\
\overset{(i)}{=}\ &\prod_{j=1}^K\pr{\abs{\hat{\mu}_j-\mu_j}<\epsilon}\\
\overset{(ii)}{=}\ &\prod_{j=1}^K\curl{1-2\times \Phi\parenth{ -\frac{\epsilon \sqrt{n_j}  }{ \sigma_j } } }\\
\overset{(iii)}{\geq}\ &\curl{1-2\times \Phi\parenth{ -\frac{ \epsilon\sqrt{\min_j n_j}  }{ \max_j\sigma_j } } }^K\\
\overset{(iv)}{=}\ &1-\alpha.
\end{aligned}\]

Equality in (i) holds if the sample means are independent of each other. 

Equality in (ii) holds if we standardize $\hat{\mu}_j-\mu_j$. There, $\Phi$ denotes the standard Normal CDF. 

Moreover, the inequality in (iii) follows because 
\[
\Phi\parenth{ -\frac{\epsilon \sqrt{n_j}  }{ \sigma_j }}\ \leq\ \Phi\parenth{ -\frac{ \epsilon\sqrt{\min_j n_j}  }{ \max_j\sigma_j } },
\]
by the monotonic increasing property of CDFs. 

Next, the equality in (iv) holds if we choose
\begin{equation}\label{eqn:margESampleSize}
\epsilon\ \dot{=}\ z_{\alpha_0/2}\times \frac{\max_j\sigma_j}{ \sqrt{\min_j n_j} }.
\end{equation}
There, $z_{\alpha_0/2}$ is the $(1-\alpha_0/2)^{\rm th}$ quantile of the standard Normal distribution. Meanwhile,
\[
\alpha_0\ \dot{=}\ 1-(1-\alpha)^{\frac{1}{K}}, 
\]
which is a simple algebraic trick to allow the desired $(1-\alpha)\times\%100$ confidence in the joint inference of the means.

For a desired margin of error $\epsilon>0$, we can simply solve for $\min_j n_j$ in Equation \eqref{eqn:margESampleSize} to see that the minimum sample size must be at least:

\[
\min_j n_j\ \geq\ \parenth{\frac{z_{\alpha_0/2}\times \max_j\sigma_j}{\epsilon}}^2.
\]

\end{itemize}

\end{proof}

\end{lemma}

\begin{lemma}[Joint estimation of independent group means nested within clusters\label{lem:jointEstMeansClusters}]\ \\

Consider a sample of $n_{jl}$ iid random variables $Y_{1jl},\dots,Y_{n_jj l}$, across $j=1,\dots,K$ and $l=1,\dots,L$. 

Let $\epsilon\in(0,1)$ and $\alpha\in(0,1)$. Denote
\[
\mu_{jl}= \E{Y_{ijl}}.
\]

\underline{\textbf{Group Means in Randomly Selected Cluster:}}

Suppose that $C\sim \text{Multinoulli}(\pi_1,\dots,\pi_L)$ with unknown 
\[
\pi_l=\pr{C=l}>0;\ l=1,\dots,L.
\]

\begin{itemize}
    \item Suppose that 
\[
\pr{a\leq Y_{ijl}\leq b }=1.
\] 
Then for
\begin{equation}\label{eqn:jointMeanEstRandCluster}
\pr{ \max_{j}\abs{ \frac{1}{n_{jC} }\sum_{i=1}^{n_j}Y_{ijC}-\mu_{jC} }<\epsilon }\geq1-\alpha,
\end{equation}
it is sufficient that
\[
\min_{j,l} n_{jl}\geq \frac{ \log\parenth{ \frac{2}{ 1-(1-\alpha)^{\frac{1}{K}} } }\abs{b-a}^2 }{2\epsilon^2}.
\]

\item In addition to
\[
\pr{a\leq Y_{ijl}\leq b }=1,
\]
suppose further that
\[
\var{Y_{ijl} }\leq \sigma^2\ \text{for each }j=1,\dots,K,\ l=1,\dots,L.
\]
Then for Equation \eqref{eqn:jointMeanEstRandCluster}, it is sufficient that
\[\begin{aligned}
&\min_{j,l} n_{jl}&\geq\ & \frac{ \log\parenth{ \frac{2}{ 1-(1-\alpha)^{\frac{1}{K}} } }\max\{a^2,b^2\} }{ \sigma^2 \curl{ \parenth{1+s }\log  \parenth{1+s } -s } },\\
\end{aligned}\]

where \[
s = \frac{ \epsilon\max\{|a|,|b|\} }{\sigma^2}.
\]

\item Suppose instead that the Central Limit Theorem (CLT) holds:
\[
\hat{\mu}_{jl}\sim\mathcal{N}\parenth{\mu_{jl},\frac{\var{Y_{ijl}}}{n_{jl} } }\ \text{for each }j=1,\dots,K,\ l=1,\dots,L.
\]
Then for Equation \eqref{eqn:jointMeanEstRandCluster}, it is sufficient that
\[
\min_{j,l} n_{jl}\ \geq\ \parenth{\frac{z_{\alpha_0/2}\times \max_{j,l}\sigma_{j,l} }{\epsilon}}^2
\]
with 
\[
\alpha_0\ \dot{=}\ 1-(1-\alpha)^{ \frac{1}{K} }.
\]

\end{itemize}

\underline{\textbf{Group Means Across All Clusters:}}

\begin{itemize}
    
\item Suppose that 
\[
\pr{a\leq Y_{ijl}\leq b }=1.
\] 
Then for
\begin{equation}\label{eqn:jointMeanEstAllClusters}
\pr{ \max_{j,l}\abs{ \frac{1}{n_{jl} }\sum_{i=1}^{n_j}Y_{ijl}-\mu_{jl} }<\epsilon }\geq1-\alpha,
\end{equation}
it is sufficient that
\[
\min_{j,l} n_{jl}\geq \frac{ \log\parenth{ \frac{2}{ 1-(1-\alpha)^{ \frac{1}{KL} } } }\abs{b-a}^2 }{2\epsilon^2}.
\]

\item In addition to
\[
\pr{a\leq Y_{ijl}\leq b }=1,
\]
suppose further that
\[
\var{Y_{ijl} }\leq \sigma^2\ \text{for each }j=1,\dots,K,\ l=1,\dots,L.
\]
Then for Equation \eqref{eqn:jointMeanEstAllClusters}, it is sufficient that
\[\begin{aligned}
&\min_{j,l} n_{jl}&\geq\ & \frac{ \log\parenth{ \frac{2}{ 1-(1-\alpha)^{ \frac{1}{KL} } } }\max\{a^2,b^2\} }{ \sigma^2 \curl{ \parenth{1+s }\log  \parenth{1+s } -s } },\\
\end{aligned}\]

where \[
s = \frac{ \epsilon\max\{|a|,|b|\} }{\sigma^2}.
\]

\item Suppose instead that the Central Limit Theorem (CLT) holds:
\[
\hat{\mu}_{jl}\sim\mathcal{N}\parenth{\mu_{jl},\frac{\var{Y_{ijl}}}{n_{jl} } }\ \text{for each }j=1,\dots,K,\ l=1,\dots,L.
\]
Then for Equation \eqref{eqn:jointMeanEstAllClusters}, it is sufficient that
\[
\min_{j,l} n_{jl}\ \geq\ \parenth{\frac{z_{\alpha_0/2}\times \max_{j,l}\sigma_{j,l}}{\epsilon}}^2
\]
with 
\[
\alpha_0\ \dot{=}\ 1-(1-\alpha)^{ \frac{1}{KL} }.
\]

\end{itemize}

\begin{proof}\ \\

\underline{\textbf{Group Means in Randomly Selected Cluster:}}\\\

In all cases, we have:
\[\begin{aligned}
&\pr{ \max_{j}\abs{ \frac{1}{n_{jC} }\sum_{i=1}^{n_j}Y_{ijC}-\mu_{jC} }<\epsilon }\\
=\ &\sum_{l=1}^L\prc{ \max_{j}\abs{ \frac{1}{n_{jl} }\sum_{i=1}^{n_j}Y_{ijl}-\mu_{jl} }<\epsilon }{C=l}\pr{C=l}\\
=\ &\sum_{l=1}^L\pr{ \max_{j}\abs{ \frac{1}{n_{jl} }\sum_{i=1}^{n_j}Y_{ijl}-\mu_{jl} }<\epsilon }\pr{C=l}\\
\geq\ &\sum_{l=1}^L\parenth{1-\alpha}\pr{C=l}\\
=\ &1-\alpha.
\end{aligned}\]

With the pertinent minimum sufficient sample size constraint, the inequality is due to an invocation of Lemma \ref{lem:jointEstMeans} for the term
\[
\pr{ \max_{j}\abs{ \frac{1}{n_{jl} }\sum_{i=1}^{n_j}Y_{ijl}-\mu_{jl} }<\epsilon };\ l=1,\dots,L
\]
in the law of total probability decomposition of the initial statement.

\underline{\textbf{Group Means Across All Clusters:}}\\

Based on the independence assumption, we have a total of $K\times L$ independent sample means
\[
\hat{\mu}_{jl};\ j=1,\dots,K,\ l=1,\dots,L.
\]
So we invoke Lemma \ref{lem:jointEstMeans} with the pertinent minimum sufficient sample size constraint and the substitution $K\mapsto K\times L$ for the total number of groups.

\end{proof}

\end{lemma}

\end{document}